\documentclass[twocolumn,twoside]{IEEEtran}
\usepackage{setspace,cite}

\usepackage[dvips]{graphicx}
\usepackage{multirow,multicol}
\usepackage[table]{xcolor}
\usepackage{ctable}

\usepackage[tbtags]{amsmath}
\usepackage{amsbsy}
\usepackage{amssymb}
\usepackage{amsfonts}

\usepackage{graphicx}
\usepackage{url}
\usepackage{algorithmic}
\usepackage{algorithm}
\usepackage{balance}
\usepackage{rotating}
\usepackage{multirow}
\usepackage{subcaption}
\usepackage{fancyvrb}
\usepackage{latexsym}
\usepackage{verbatim}

\newtheorem{proposition}{Proposition}

\newtheorem{remark}{Remark}

\newtheorem{example}{Example}

{\begin{list}               
    {$\bullet$ \hfill}{
        \setlength{\leftmargin}{\parindent}
        \setlength{\parsep}{0.04\baselineskip}
        \setlength{\itemsep}{0.5\parsep}
        \setlength{\labelwidth}{\leftmargin}
        \setlength{\labelsep}{0em}}
    }
{\end{list}}

\providecommand{\eref}[1]{\eqref{#1}}  
\providecommand{\cref}[1]{Chapter~\ref{#1}}

\providecommand{\fref}[1]{Figure~\ref{#1}}

\providecommand{\R}{\ensuremath{\mathbb{R}}}

\providecommand{\bydef}{\overset{\text{def}}{=}}

\renewcommand{\vec}[1]{\ensuremath{\boldsymbol{#1}}}
\providecommand{\mat}[1]{\ensuremath{\boldsymbol{#1}}}

\providecommand{\calD}{\mathcal{D}}

\providecommand{\calL}{\mathcal{L}}

\providecommand{\calN}{\mathcal{N}}

\providecommand{\mA}{\mat{A}}

\providecommand{\mC}{\mat{C}}

\providecommand{\mI}{\mat{I}}

\providecommand{\mK}{\mat{K}}
\providecommand{\mL}{\mat{L}}

\providecommand{\mR}{\mat{R}}
\providecommand{\mS}{\mat{S}}

\providecommand{\mU}{\mat{U}}
\providecommand{\mV}{\mat{V}}
\providecommand{\mW}{\mat{W}}

\providecommand{\vp}{\vec{p}}
\providecommand{\vq}{\vec{q}}

\providecommand{\vu}{\vec{u}}
\providecommand{\vv}{\vec{v}}

\providecommand{\vx}{\vec{x}}
\providecommand{\vy}{\vec{y}}

\providecommand{\mSigma}{\mat{\Sigma}}

\providecommand{\veta}{\vec{\eta}}

\providecommand{\vvtilde}{\boldsymbol{\widetilde{v}}}
\providecommand{\vxtilde}{\boldsymbol{\widetilde{x}}}

\providecommand{\vvhat}{\boldsymbol{\widehat{v}}}
\providecommand{\vubar}{\boldsymbol{\bar{u}}}

\newcommand{\subjectto}{\mathop{\mathrm{subject\, to}}}
\newcommand{\argmin}[1]{\mathop{\underset{#1}{\mbox{argmin}}}}

\newcommand{\minimize}[1]{\mathop{\underset{#1}{\mathrm{minimize}}}}

\newcommand{\diag}[1]{\mathop{\mathrm{diag}\left\{#1\right\}}}

\title{Algorithm-Induced Prior for Image Restoration}
\author{Stanley H. Chan,~\IEEEmembership{Member,~IEEE}
\thanks{The authors is with School of ECE and Dept of Statistics, Purdue University, West Lafayette, IN 47907. Email: stanleychan@purdue.edu.}}

\graphicspath{{pix/}}

\begin{document}
\maketitle
\begin{abstract}
This paper studies a type of image priors that are constructed implicitly through the alternating direction method of multiplier (ADMM) algorithm, called the algorithm-induced prior. Different from classical image priors which are defined before running the reconstruction algorithm, algorithm-induced priors are defined by the denoising procedure used to replace one of the two modules in the ADMM algorithm. Since such prior is not explicitly defined, analyzing the performance has been difficult in the past.

Focusing on the class of symmetric smoothing filters, this paper presents an explicit expression of the prior induced by the ADMM algorithm. The new prior is reminiscent to the conventional graph Laplacian but with stronger reconstruction performance. It can also be shown that the overall reconstruction has an efficient closed-form implementation if the associated symmetric smoothing filter is low rank. The results are validated with experiments on image inpainting.
\end{abstract}

\begin{keywords}
Image reconstruction, image denoising, graph Laplacian, ADMM, symmetric smoothing filters.
\end{keywords}

\section{Introduction}
\subsection{Background, Scope, and Related Work}
Alternating direction methods of multipliers (ADMM) is perhaps the most popular algorithm for solving linear inverse problems in recent years, particularly for image restoration \cite{Yang_Zhang_Yin_2009,Afonso_Bioucas_Figueiredo_2010}. Despite different perspectives of the algorithm, (e.g., operator splitting \cite{Eckstein_Bertsekas_1992}, proximal methods \cite{Parikh_Boyd_2014}, split Bregman \cite{Esser_2009}, to name a few,) the common principle behind ADMM is to convert the original minimization of the form
\begin{equation}
\minimize{\vx \in \R^n} \;\; f(\vx) + \lambda s(\vx)
\label{eq:main 1}
\end{equation}
into an equivalent constrained problem
\begin{equation}
\begin{array}{ll}
\minimize{\vx \in \R^n,\vv \in \R^n} &\;\; f(\vx) + \lambda s(\vv)\\
\subjectto &\;\; \vx = \vv,
\end{array}
\label{eq:main 2}
\end{equation}
and solve the constrained problem by alternatingly minimizing the augmented Lagrangian function. Under mild conditions, e.g., when $f(\cdot)$ is strongly convex and $s(\cdot)$ is convex, the convergence of the algorithm is typically guaranteed \cite{Boyd_Parikh_Chu_Peleato_Eckstein_2011}.

In setting up the optimization problem in \eref{eq:main 1}, the objective function $f(\cdot)$ and the regularization function $s(\cdot)$ are almost always fixed \emph{before} running the algorithm. For example, when solving a non-blind deblurring problem using a total variation regularization \cite{Chan_Khoshabeh_Gibson_2011}, the functions $f(\cdot)$ and $s(\cdot)$ are
\begin{equation}
f(\vx) = \|\mA\vx - \vy\|^2, \quad s(\vx) = \|\vx\|_{TV},
\label{eq:example1}
\end{equation}
where $\mA$ is the blur operator, $\vy$ is the observed image, and $\|\vx\|_{TV}$ is the total variation norm of the image $\vx$.

For most image restoration problems, $f(\vx)$ is chosen according to the forward imaging model, and is fixed as long as we agree with the forward model. But $s(\vx)$ is the user's subjective belief of how the solution should look like, a.k.a. the \emph{prior}. In literature, apart from the total variation prior mentioned in \eref{eq:example1}, there are enormous number of priors we can use. However, there is one thing in common, which is that $s(\vx)$ has to be defined before using the ADMM algorithm.

In this paper, I present an ADMM algorithm where the regularization function $s(\vx)$ is \emph{unknown} a-priori. At a first glance, this might seem unnatural because if $s(\vx)$ is unknown, then it is unclear about what we are trying to optimize in \eref{eq:main 1}. However, as will be discussed shortly, the ADMM algorithm can generally be written as two \emph{modules} -- an inverse module, and a denoising module. The idea is to replace the denoising module by some off-the-shelf image denoising algorithm, e.g., non-local means \cite{Buades_Coll_2005_Journal,Chan_Zickler_Lu_2014} or BM3D \cite{Dabov_Foi_Katkovnik_2007}. In other words, we do not explicitly define $s(\vx)$ before running the algorithm, but use a denoising algorithm to perform the role of $s(\vx)$.

Replacing the denoising module of the ADMM algorithm by an off-the-shelf denoising algorithm was first proposed by Bouman and colleagues \cite{Venkatakrishnan_Bouman_Wohlberg_2013}, to the best of my knowledge. Perhaps of the heuristic nature of the method, they call it the ``Plug-and-Play'' algorithm to stress that one can plug in any denoising algorithm and get the ADMM algorithm running. Under appropriate conditions on the denoising algorithm, one can prove the convergence of the plug-and-play \cite{Sreehari_Venkatakrishnan_Wohlberg_2015}.

In the context of compressive sensing \cite{Metzler_Maleki_Baraniuk_2014,Ma_Zhu_Baron_2015,Dar_Bruckstein_Elad_2015}, a similar version of the plug-and-play is also being studied. In \cite{Metzler_Maleki_Baraniuk_2014}, Baranuik and colleagues considered an approximate message passing (AMP) algorithm for recovering images. Recognizing that AMP also has a ``inverse module'' and a ``denoising module'', they replace the shrinkage step in the conventional AMP with an off-the-shelf denoising algorithm (BM3D in their paper). Again, under appropriate conditions of the denoising algorithm, they proved that the AMP converges.

\subsection{Contributions}
The focus of this paper is not to find weaker conditions under which ``plug-and-play'' converges. Rather, I like to address another equally important question: What is the original prior $s(\vx)$ if we choose a particular denoising algorithm?  Answering this problem is essential to understand this type of algorithms in general. To make the discussion concrete, I will focus on the class of symmetric smoothing filters \cite{Milanfar_2013b,Chan_Zickler_Lu_2015,Chan_Zickler_Lu_2016} which is broad enough to include many denoising methods such as bilateral filter, non-local means \cite{Buades_Coll_2005_Journal} and LARK \cite{Milanfar_2013a}, but at the same time also allows us to exploit matrix structures, e.g., the graph Laplacian \cite{Meyer_Shen_2012}. I call the new prior as an \emph{algorithm-induced prior} to reflect the algorithmic nature of the prior.

The rest of the paper is organized as follows. I will first setup the problem in Section~\ref{sec:problem}. Then, in Section~III, I will address the question about the original prior of the ADMM-induced algorithm, and discuss linkages with the conventional graph Laplacian prior. Experimental results are presented in Section IV, and a conclusion is given in Section V.

\section{Concept of Algorithm-Induced Prior}
\label{sec:problem}
\subsection{ADMM Algorithm}
To begin the discussion I will first briefly introduce the ADMM algorithm. Interested readers can read \cite{Boyd_Parikh_Chu_Peleato_Eckstein_2011} for additional technical details.

Given the constrained minimization \eref{eq:main 2}, the ADMM algorithm defines the augmented Lagranian function as
\begin{equation}
\calL(\vx,\vv,\vu) = f(\vx) + \lambda s(\vv) + \vu^T(\vx-\vv) + \frac{\rho}{2}\|\vx - \vv\|^2.
\label{eq:augmented lagrangian function}
\end{equation}
where $\vu \in \R^n$ is the Lagrange multiplier, and $\rho > 0$ is the half quadratic penalty parameter. The algorithm then proceeds to update each variable as follows
\begin{align}
\vx^{(k+1)} &= \argmin{\vx\in \R^n} \;\; \calL(\vx, \; \vv^{(k)}, \; \vu^{(k)}), \label{eq:ADMM1_x}\\
\vv^{(k+1)} &= \argmin{\vv\in \R^n} \;\; \calL(\vx^{(k+1)}, \; \vv, \; \vu^{(k)}), \label{eq:ADMM1_v}\\
\vu^{(k+1)} &= \vu^{(k)} + \rho(\vx^{(k+1)} - \vv^{(k+1)}). \label{eq:ADMM1_u}
\end{align}
The minimizations in \eref{eq:ADMM1_x} and \eref{eq:ADMM1_v} are known as the primal updates, whereas the descent step in \eref{eq:ADMM1_u} is the dual update. If both $f(\cdot)$ and $s(\cdot)$ are closed, proper and convex, and if $\calL(\cdot)$ has a saddle point, then one can prove convergence of the ADMM algorithm in terms of primal residue, primal objective and dual variable \cite{Boyd_Parikh_Chu_Peleato_Eckstein_2011}. In case when \eref{eq:ADMM1_x} and \eref{eq:ADMM1_v} are solved simultaneously instead of sequentially as presented above, then one will obtain the augmented Lagrangian method (ALM).

With some manipulations and rearrangement of terms we can show the following.
\begin{proposition}
\label{prop 1}
The iterations \eref{eq:ADMM1_x}-\eref{eq:ADMM1_u} are equivalent to
\begin{align}
\vx^{(k+1)} &= \argmin{\vx\in \R^n} \;\; f(\vx) +\frac{\rho}{2} \|\vx - \vxtilde^{(k)}\|^2, \label{eq:ADMM2,x}\\
\vv^{(k+1)} &= \argmin{\vv\in \R^n} \;\; \lambda s(\vv) + \frac{\rho}{2}\|\vv - \vvtilde^{(k)}\|^2,\label{eq:ADMM2,v}\\
\vubar^{(k+1)} &= \vubar^{(k)} + (\vx^{(k+1)} - \vv^{(k+1)}),\label{eq:ADMM2,u}
\end{align}
where $\vubar^{(k)} \bydef \frac{1}{\rho}\vu^{(k)}$ is the scaled multiplier, $\vxtilde^{(k)} \bydef \vv^{(k)}-\vubar^{(k)}$ and $\vvtilde^{(k)} \bydef \vx^{(k+1)}+\vubar^{(k)}$.
\end{proposition}

The proof is skipped because it is essentially completing squares. The reason of rewriting the ADMM as above is to demonstrate the \emph{modular} structure of the ADMM algorithm which we shall discuss shortly. As a side remark, the iterations \eref{eq:ADMM2,x}-\eref{eq:ADMM2,u} can be defined as a \emph{proximal algorithm} \cite{Parikh_Boyd_2014}.

To gain more insights into the modular structure presented in Proposition~\ref{prop 1}, let us consider the following example.
\begin{example}
If we use $f(\cdot)$ and $s(\cdot)$ given in \eref{eq:example1}, we observe that \eref{eq:ADMM2,x} and \eref{eq:ADMM2,v} become
\begin{align*}
\vx^{(k+1)} &= \argmin{\vx} \;\; \|\mA\vx - \vy\|^2 + \frac{\rho}{2} \|\vx - \vxtilde^{(k)}\|^2, \\
\vv^{(k+1)} &= \argmin{\vv} \;\; \lambda \|\vx\|_{TV} + \frac{\rho}{2}\|\vv - \vvtilde^{(k)}\|^2,
\end{align*}
which is a reconstruction problem with a quadratic regularization, and a denoising problem (to denoise $\vvtilde^{(k)}$) with a total variation regularization, respectively.
\end{example}

%

\subsection{Algorithm-Induced Prior}
Recognizing the ``denoising'' module in the ADMM algorithm, we replace the $\vv$-subproblem by a denoising algorithm. Formally, if we denote $\calD_{h}$ as the denoising algorithm, i.e.,
\begin{equation}
\calD_{h}\left(\vvtilde^{(k)} \right) \bydef \argmin{\vv\in \R^n} \;\; \lambda s(\vv) + \frac{\rho}{2}\|\vv - \vvtilde^{(k)}\|^2,
\end{equation}
where the subscript $h > 0$ specifies the internal parameter of the denoising algorithm, then the $\vv$-subproblem becomes
\begin{equation}
\vv^{(k+1)} = \calD_{h}\left(\vvtilde^{(k)} \right).
\end{equation}
For example, if we choose a symmetric smoothing filter, then the denoising algorithm $\calD_{h}: \; \R^n \rightarrow \R^n$ takes the form
\begin{equation}
\calD_{h}\left(\vvtilde^{(k)} \right) = \mW_h^{(k)} \vvtilde^{(k)},
\label{eq:symmetric smoothing filter}
\end{equation}
where $\mW_h^{(k)} \in \R^{n \times n}$ is a doubly stochastic matrix. An explicit example of $\mW_h^{(k)}$ using the non-local means is illustrated in the following example.

\begin{example}
Consider the non-local means \cite{Buades_Coll_2005_Journal} as an example. One can first construct a kernel matrix $\mK_h^{(k)}$ with the $(i,j)$-th entry
\begin{equation*}
\left[\mK_h^{(k)}\right]_{ij} \bydef \exp\left\{-\| \vvtilde^{(k)}_i - \vvtilde^{(k)}_j \|^2/ (2h^2)\right\},
\end{equation*}
where $\vvtilde^{(k)}_i$ denotes the $i$-th patch of the input $\vvtilde^{(k)}$. Then, by applying Sinkhorn-Knopp balancing algorithm \cite{Sinkhorn_Knopp_1967} one can determine a pair of diagonal matrices $\mR^{(k)}$ and $\mC^{(k)}$ such that
$$
\mW^{(k)}_h \bydef \mR^{(k)} \mK^{(k)} \mC^{(k)}
$$
is a doubly stochastic matrix.
\end{example}

\section{Analysis of Algorithm-Induced Prior}
In this section I will address the question: What is the original prior $s(\vx)$ if we choose $\calD_h$ as a symmetric smoothing filter? For notational simplicity I will drop the scripts $(\cdot)^{(k)}$ and $(\cdot)_h$.

\subsection{Original $s(\vx)$}
The first main result is stated in the following proposition, which provides an explicit formula for the regularization $s(\vx)$ when symmetric smoothing filters are used.

\begin{proposition}
\label{prop: s}
Let $\mW$ be a symmetric smoothing filter with $\mathrm{rank}(\mW) = m$, and let $\mW^+$ be the pseudo-inverse of $\mW$, i.e., $\mW^+ = \mU\mSigma_+^{-1}\mU^T$ where $\mU$ and $\mSigma = \diag{s_1,\ldots,s_m,0,\ldots,0}$ are the eigenvectors and eigenvalues of $\mW$ respectively, and $\mSigma_+^{-1} = \diag{1/s_1,\ldots,1/s_m,0,\ldots,0}$. For a fixed $\mW$, if
\begin{equation}
s(\vv) = \frac{\rho}{2\lambda} \vv^T (\mI-\mW)\mW^{+} \vv,
\label{eq: nlm s}
\end{equation}
then
\begin{equation}
\argmin{\vv} \;\; \lambda s(\vv) + \frac{\rho}{2}\|\vv - \vvtilde\|^2 = \mW\vvtilde \;\bydef\; \vvhat.
\label{eq: nlm v sub}
\end{equation}
\end{proposition}
\begin{proof}
There are two ways of proving this proposition. The first way is a ``reverse engineering'' approach. By plugging \eref{eq: nlm s} into \eref{eq: nlm v sub} and setting the first order derivative to zero we can show that $\vvhat = \mW\vvtilde$.

The alternative proof is a constructive one. First, we observe that in order to obtain $\mW\vvtilde$ on the right hand side of \eref{eq: nlm v sub}, we must have $s(\vv)$ being quadratic. Therefore, we let
\begin{equation*}
s(\vv) = \alpha \vv^T \mC \vv,
\end{equation*}
for some symmetric matrix $\mC$ and constant $\alpha$. Taking the first order derivative of the resulting function yields
\begin{align*}
\frac{d}{d\vv} \left( \lambda s(\vv) + \frac{\rho}{2}\|\vv - \vvtilde\|^2 \right)= 2 \lambda \alpha \mC\vv + \rho(\vv-\vvtilde) = 0.
\end{align*}
Rearranging the terms, we obtain a linear equation
\begin{equation*}
\left( \frac{2\lambda}{\rho} \alpha \mC + \mI\right)\vv = \vvtilde.
\end{equation*}
Since $\alpha$ can be arbitrary, we set $\alpha = \rho/(2\lambda)$. Consequently, we have $(\mC+\mI)\vv = \vvtilde$. Multiplying both sides with $\mW$ yields $\mW(\mC+\mI)\vv = \mW\vvtilde$. Thus, in order to obtain $\vv = \mW\vvtilde$, $\mC$ must be chosen such that
\begin{equation*}
\mW(\mC+\mI) = \mI,
\end{equation*}
which gives $\mC = (\mI-\mW)\mW^+$.
\end{proof}

There are a few important properties of the regularization $s(\vv)$ shown in \eref{eq: nlm s}. First, for any fixed $\mW$, the matrix $(\mI-\mW)\mW^+$ is symmetric positive semidefinite. In fact, since the eigenvalues of a symmetric smoothing filter $\mW$ is always bounded between 0 and 1, i.e., $0 \preceq \mSigma \preceq 1$, it holds that the eigenvalues of $(\mI-\mW)\mW^+$ is also bounded between 0 and 1. Therefore, if $\mW$ is pre-defined before running the ADMM algorithm, then $s(\vx)$ is convex and hence the overall optimization is also convex.

Second, if we compare \eref{eq: nlm s} with the conventional graph Laplacian regularization $s(\vv) = \vv^T \mL \vv$ in the literature \cite{Shuman_Narang_Frossard_2013}, where $\mL \bydef \mI - \mW$, we observe that \eref{eq: nlm s} has an additional term $\mW^+$. Using a graph signal processing terminology, we can view $\mW$ as a lowpass filter and $\mL$ is a highpass filter. $\mW^+$ is a bandpass filter because of the truncation property of the pseudo-inverse. Therefore, the regularization $\vv^T(\mI-\mW)\mW^+\vv$ penalizes a smaller (but more focused) set of graph frequencies than the conventional regularization $\vv^T\mL\vv$. In Section IV we will compare the performance.

\subsection{Closed-form Solution}
Proposition \ref{prop: s} suggests a new prior which deserves a closer look. First of all, assume, for simplicity, that the matrix $\mW$ is \emph{fixed} throughout the ADMM iteration. This can be done either in an oracle setting (i.e., find $\mW$ from the ground truth solution), or in a pre-filtering setting (i.e., find $\mW$ from some initial guess of the solution). Both ways are common in image restoration \cite{Talebi_Milanfar_2014,Talebi_Zhu_Milanfar_2013}.

Substituting $f(\vx) = \frac{1}{2}\|\mA\vx-\vy\|^2$ and the specific prior $s(\vx)$ given by \eref{eq: nlm s} into the original optimization \eref{eq:main 1}, the problem becomes
\begin{equation}
\minimize{\vx \in \R^n} \;\; \varphi(\vx) \bydef \frac{1}{2}\|\mA\vx-\vy\|^2 + \frac{\rho}{2} \vx^T (\mI-\mW)\mW^{+} \vx,
\label{eq:main new}
\end{equation}
which is a quadratic optimization. Closed-form solution of \eref{eq:main new} exists, and is given by solving the normal equation
\begin{equation}
\left(\mA^T\mA + \rho(\mI-\mW)\mW^{+} \right)\vx = \mA^T\vy.
\label{eq:main new solution}
\end{equation}

Since closed-form solution exists, it is possible to bypass the ADMM iterations and obtain the solution efficiently. However, from a computational perspective, there are two issues of \eref{eq:main new solution} which we need to overcome. First, \eref{eq:main new solution} involves inverting an $n \times n$ matrix which is computationally prohibitive for large $n$. Second, if $\mW$ has a full rank but with some very small eigenvalues, $\mW^{+}$ will cause numerical instability, depending on the numerical threshold for truncating the eigenvalues. Therefore, if we want to use the closed form solution in \eref{eq:main new solution}, one possible approach is to bypass the pseudo-inverse $\mW^+$. This can be done using the following algebraic trick.

\begin{proposition}
\label{prop 3}
The solution of \eref{eq:main new} is given by
\begin{equation}
\vx = \mU\mSigma\left(\mSigma\mU^T\mA^T\mA\mU\mSigma  + \rho \mSigma(\mI-\mSigma)\right)^{+} \mSigma\mU^T\mA^T\vy,
\label{eq:main new solution 2}
\end{equation}
where $\mW = \mU\mSigma\mU^T$ is the eigen-decomposition of $\mW$.
\end{proposition}

\begin{proof}
Define $\vp \bydef \mW^{+}\vx$ (or, equivalently, $\vx = \mW\vp$). Then it holds that
\begin{align*}
\varphi(\vp) = \frac{1}{2}\|\mA\mW\vp-\vy\|^2 + \frac{\rho}{2} \vp^T \mW(\mI-\mW)\vp,
\end{align*}
because $\mW = \mW^T$. Consider the eigen-decomposition $\mW = \mU \mSigma \mU^T$, and let $\vq = \mU^T\vp$, it follows that
\begin{align*}
\varphi(\vq) = \frac{1}{2}\|\mA\mU\mSigma\vq-\vy\|^2 + \frac{\rho}{2} \vq^T \mSigma(\mI-\mSigma)\vq.
\end{align*}
The minimizer of this quadratic function is given by the solution of the normal equation
\begin{equation*}
\left(\mSigma\mU^T\mA^T\mA\mU\mSigma  + \rho \mSigma(\mI-\mSigma)\right) \vq = \mSigma\mU^T\mA^T\vy.
\end{equation*}
Since $\mU\mU^T =\mI$, it holds that $\vp = \mU\vq$ and hence the solution is $\vx = \mW\vp = \mW\mU\vq = \mU\mSigma\vq$.
\end{proof}

The importance of Proposition~\ref{prop 3} is that \eref{eq:main new solution 2} only involves one pseudo-inverse whereas \eref{eq:main new solution} requires two pseudo-inverses. Computing \eref{eq:main new solution 2} is numerically easy. If $\mathrm{rank}(\mW) = m$, we decompose $\mW = \mV\mS\mV^T$ where $\mV \in \R^{n \times m}$ is the truncated eigenvector, and $\mS \in \R^{m \times m}$ is the truncated eigenvalue. The matrices $\mV$ and $\mS$ can be computed using the Nystr\"{o}m approximation \cite{Talebi_Milanfar_2014}. Consequently, \eref{eq:main new solution 2} becomes
\begin{equation}
\vx = \mV\mS\left(\mS\mV^T\mA^T\mA\mV\mS  + \rho \mS(\mI-\mS)\right)^{-1} \mS\mV^T\mA^T\vy.
\label{eq:main new solution 3}
\end{equation}
Inspecting \eref{eq:main new solution 3}, we observe that the matrix inversion only involves an $m \times m$ matrix, which is significantly smaller than the $n \times n$ matrix in \eref{eq:main new solution}. The matrix $\mA\mV$ are usually not difficult to evaluate. Below are two examples for image inpainting and deblurring.

\begin{figure*}[t]
\begin{tabular}{cccc}
\includegraphics[width=0.2\linewidth]{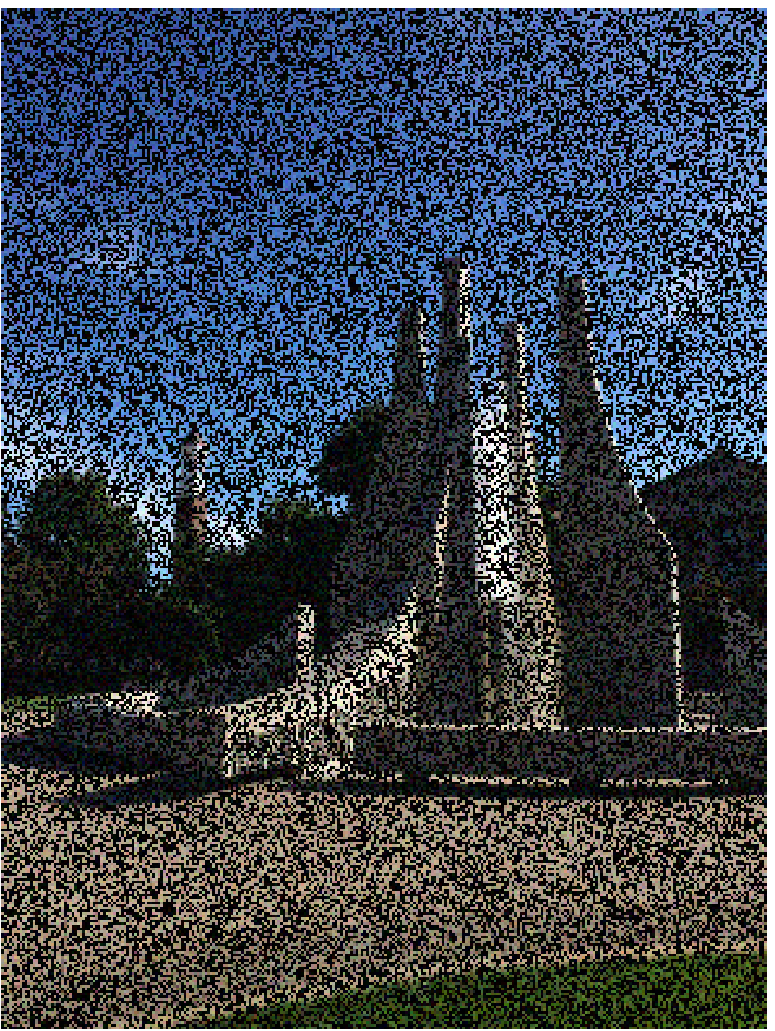}&
\includegraphics[width=0.2\linewidth]{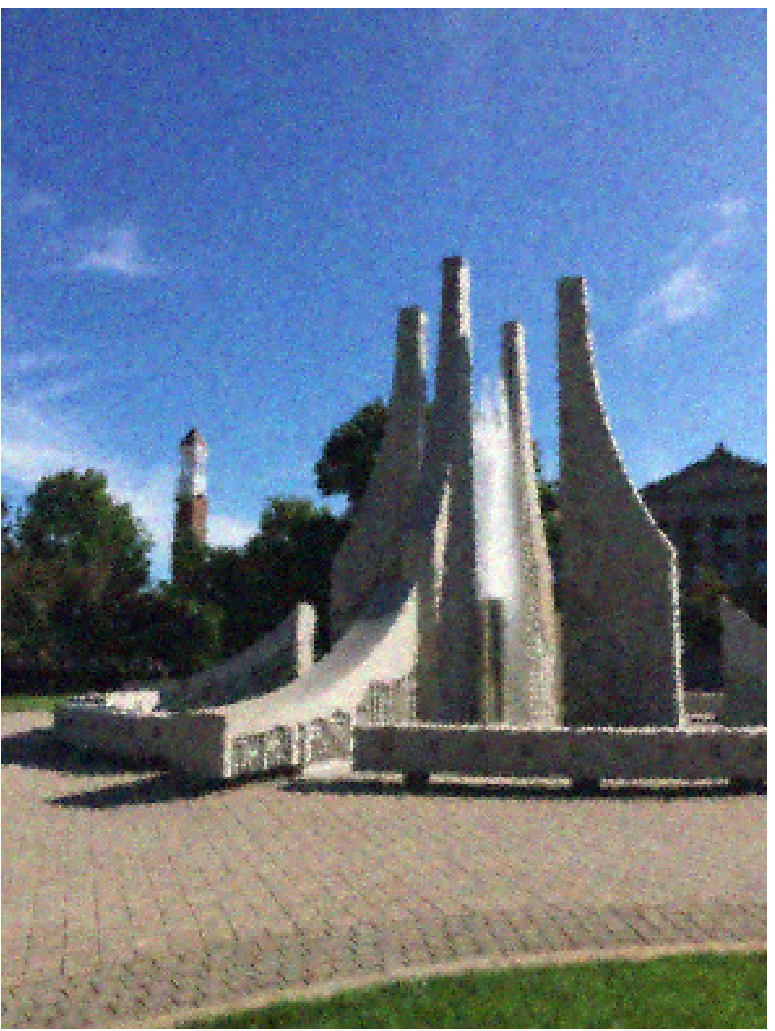}&
\includegraphics[width=0.2\linewidth]{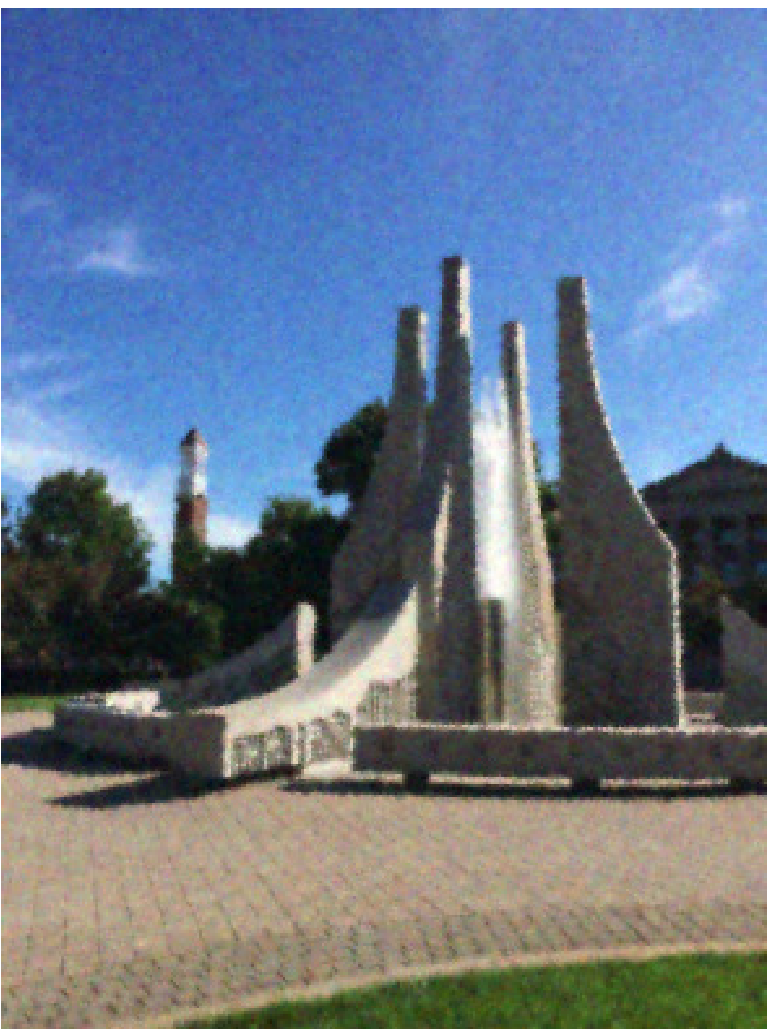}&
\includegraphics[width=0.3\linewidth]{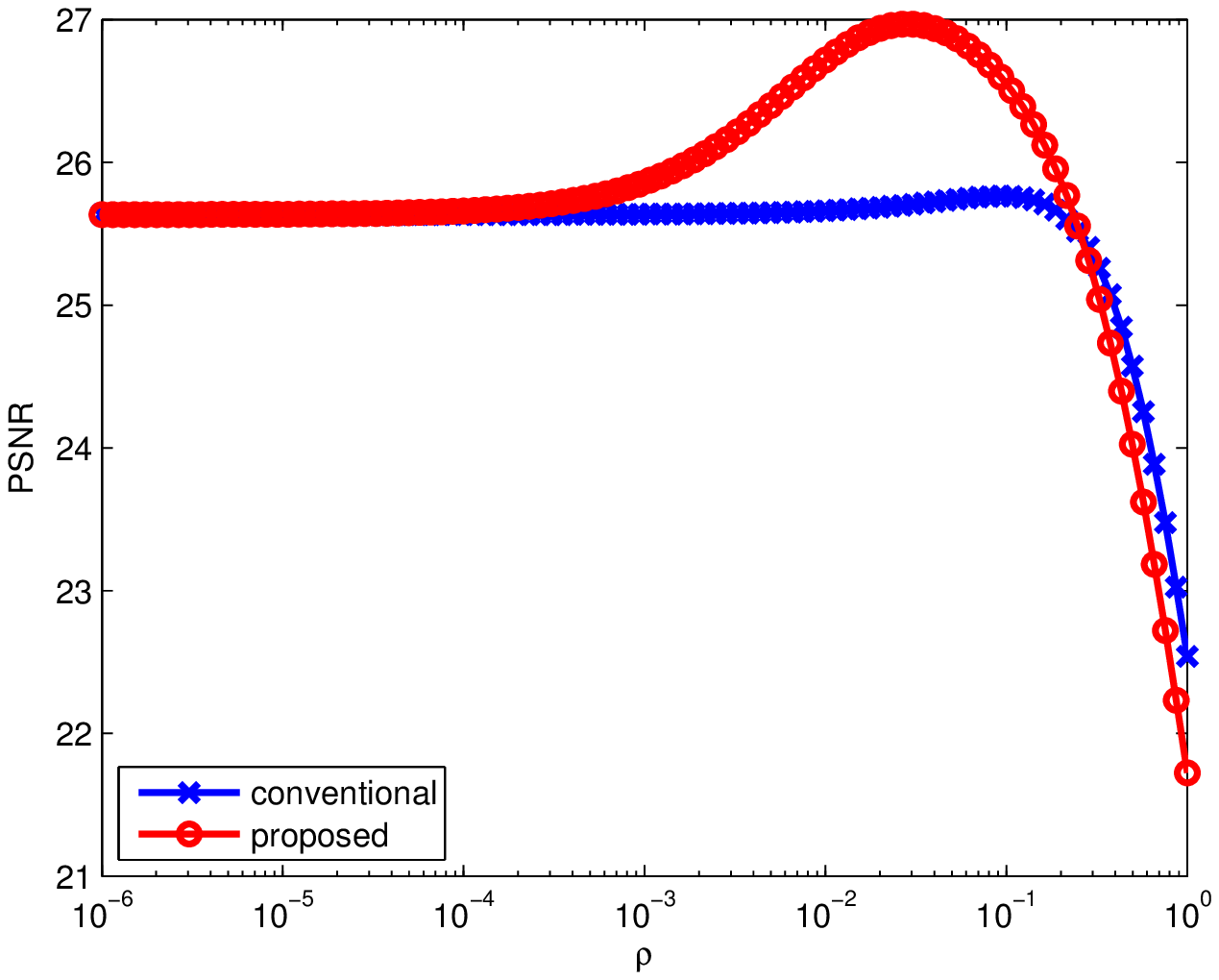}\\
(a) 50\% missing pixel & (b) conventional Laplacian & (c) proposed & (d) PSNR as a function of $\rho$
\end{tabular}
\caption{Image inpainting using (b) the conventional graph Laplacian prior $s(\vx) = \frac{\rho}{2} \vx^T(\mI-\mW)\vx$ and (c) the proposed algorithm-induced prior $s(\vx) = \frac{\rho}{2}\vx^T(\mI-\mW)\mW^{+}\vx$. For both methods, an optimal $\rho$ is selected for the best PSNR as shown in (d). The matrix $\mW$ is the non-local mean weight computed from an initial solution using Shepard's interpolation \cite{Shepard_1968}. In this example, the RGB color channels are processed independently. }
\label{fig:result}
\end{figure*}



\begin{example}
For image inpainting, the matrix $\mA$ is a binary diagonal matrix. Hence, $\mA\mV$ involves picking the non-zero columns of $\mV$.
\end{example}

\begin{example}
For image deblurring, the matrix $\mA$ is a convolution. Hence, the multiplication of $\mA$ and $\vv_i$, the $i$th column of $\mV$, is a blurring operation on $\vv_i$.
\end{example}

\vspace{2ex}

\begin{remark}
In practice, symmetric smoothing filters sometimes have very narrow spatial support and hence $\mW$ is a banded diagonal matrix. A banded diagonal $\mW$ has a significantly higher rank, making the eigen-decomposition difficult. However, the good news is that such $\mW$ is often easy to compute. In this case, the closed form should be replaced by the ADMM iteration.
\end{remark}

\section{Experimental Results}
To demonstrate the usefulness of the algorithm-induced prior, I use image inpainting as an example. In this experiment, let $\mA \in \R^{n \times n}$ be a binary diagonal matrix, where $A_{ii} \sim \mathrm{Bernoulli}(\xi)$ for a sampling ratio $0 \le \xi \le 1$. Here, $\xi = 1$ means that all pixels are acquired with probability 1, whereas $\xi = 0$ means a probability 0. The observed image is
\begin{equation*}
\vy = \mA\vx + \veta,
\end{equation*}
where $\veta \sim \calN(0,\sigma^2)$ is an additive iid Gaussian noise. In this experiment, $\sigma = 0.05$.

There are two choices of the filter $\mW$. The first choice is to compute $\mW$ from the ground truth solution. This is called the oracle setting, and is the best possible setting we can use under our framework. The second choice is to compute $\mW$ from some initial estimate of the solution. In this problem, the initial estimate is performed using the classical Shepard's interpolation method \cite{Shepard_1968}. If a more sophisticated initial estimator is used, it is likely that the performance will be improved.

We compare two graph Laplacian priors, namely
\begin{align}
\mbox{conventional:} &\quad\quad\quad\quad s_{\mL}(\vx) = \frac{\rho}{2} \vx^T\mL\vx, \notag \\
\mbox{proposed:}     &\quad\quad\quad\quad s_{\mC}(\vx) = \frac{\rho}{2} \vx^T\mC\vx, \label{eq:priors}
\end{align}
where $\mL \bydef \mI-\mW$ is the classical graph Laplacian, and $\mC \bydef (\mI-\mW)\mW^{+}$ is the proposed algorithm-induced prior. The results are shown in Table~\ref{table:psnr}. It is evident from the result that $s_{\mC}(\vx)$ performs consistently better than $s_{\mL}(\vx)$ for both the estimated $\mW$ and the oracle $\mW$. The gap is especially prominent if we look at the oracle case at 80\% missing.

\begin{table}[t]
\begin{tabular}{c|cc|cc}
\hline
\%              & \multicolumn{2}{c|}{estimated $\mW$}  & \multicolumn{2}{c}{oracle $\mW$} \\
missing         & conventional    & proposed        & conventional    & proposed \\
\hline
20\%             & 26.24dB           & 27.59dB              & 44.48dB        & 47.22dB\\
40\%             & 24.87dB           & 25.82dB              & 43.36dB        & 46.44dB\\
60\%             & 23.72dB           & 24.32dB              & 41.04dB        & 44.97dB\\
80\%             & 21.56dB           & 21.85dB              & 37.36dB        & 42.57dB\\
\hline
\end{tabular}
\caption{PSNR of ``\texttt{cameraman}'' with optimized $\rho$.}
\label{table:psnr}
\end{table}

\fref{fig:result} shows a visual comparison of an image captured by an i-Phone 6 camera with 50\% missing pixel generated by MATLAB simulation. It should be reminded that in all experiments the parameter $\rho$ are adjusted accordingly for $s_{\mL}(\vx)$ and $s_{\mC}(\vx)$. \fref{fig:result}(d) illustrates such dependence: The optimal $\rho$ are different for different priors. However, the best PSNR of $s_{\mC}(\vx)$ is significantly higher than that of $s_{\mL}(\vx)$.

\section{Conclusion}
Algorithm-induced prior is a strong performing but intriguing prior that we have little understanding about. Therefore, being able to explicitly write down the formula of the algorithm-induced prior is an important step which allows us to analyze the performance of such prior. In this paper, I demonstrated the case of symmetric smoothing filters and drew connections with the conventional graph Laplacian prior. On a set of image inpainting experiments, algorithm-induced prior offers consistently better results than the conventional graph Laplacian. As we progress along this direction, I believe that the interplay between the objective function and the denoising procedure should be studied in greater details.

\section{Acknowledgement}
I like to thank Charles Bouman for pointing me to the problem, and Suhas Sreehari for fruitful discussions. I also thank Dror Baron for telling me his work on AMP using BM3D.

\balance
\bibliographystyle{IEEEbib}
\bibliography{refs}

\begin{thebibliography}{10}

\bibitem{Yang_Zhang_Yin_2009}
J.~Yang, Y.~Zhang, and W.~Yin,
\newblock ``An efficient {TVL1} algorithm for deblurring multichannel images
  corrupted by impulsive noise,''
\newblock {\em SIAM J. on Sci. Comput.}, vol. 31, no. 4, pp. 2842--2865, Jul.
  2009.

\bibitem{Afonso_Bioucas_Figueiredo_2010}
M.~V. Afonso, J.~M. Bioucas-Dias, and M.~A.~T. Figueiredo,
\newblock ``Fast image recovery using variable splitting and constrained
  optimization,''
\newblock {\em IEEE Trans. Image Process.}, vol. 19, no. 9, pp. 2345--2356,
  Aug. 2010.

\bibitem{Eckstein_Bertsekas_1992}
J.~Eckstein and D.~P. Bertsekas,
\newblock ``On the {D}ouglas-{R}achford splitting method and the proximal point
  algorithm for maximal monotone operators,''
\newblock {\em Math. Program.}, vol. 55, no. 3, pp. 293--318, Jun. 1992.

\bibitem{Parikh_Boyd_2014}
N.~Parikh and S.~Boyd,
\newblock ``Proximal algorithms,''
\newblock {\em Foundations and Trends in Optimization}, vol. 1, no. 3, pp.
  123--231, 2014.

\bibitem{Esser_2009}
E.~Esser,
\newblock ``Applications of {Lagrangian}-based alternating direction methods
  and connections to split {Bregman },''
\newblock Tech. {R}ep., University of California, Los Angeles, 2009,
\newblock Available online: ftp://ftp.math.ucla.edu/pub/camreport/cam09-31.pdf.

\bibitem{Boyd_Parikh_Chu_Peleato_Eckstein_2011}
S.~Boyd, N.~Parikh, E.~Chu, B.~Peleato, and J.~Eckstein,
\newblock ``Distributed optimization and statistical learning via the
  alternating direction method of multipliers,''
\newblock {\em Found. Trends Mach. Learn.}, vol. 3, no. 1, pp. 1--122, Jan.
  2011.

\bibitem{Chan_Khoshabeh_Gibson_2011}
S.~H. Chan, R.~Khoshabeh, K.~B. Gibson, P.~E. Gill, and T.~Q. Nguyen,
\newblock ``An augmented {L}agrangian method for total variation video
  restoration,''
\newblock {\em IEEE Trans. Image Process.}, vol. 20, no. 11, pp. 3097--3111,
  Nov. 2011.

\bibitem{Buades_Coll_2005_Journal}
A.~Buades, B.~Coll, and J.~Morel,
\newblock ``A review of image denoising algorithms, with a new one,''
\newblock {\em SIAM Multiscale Model and Simulation}, vol. 4, no. 2, pp.
  490--530, 2005.

\bibitem{Chan_Zickler_Lu_2014}
S.~H. Chan, T.~Zickler, and Y.~M. Lu,
\newblock ``{Monte-Carlo} non-local means: Random sampling for large-scale
  image filtering,''
\newblock {\em IEEE Trans. Image Process.}, vol. 23, no. 8, pp. 3711--3725,
  Aug. 2014.

\bibitem{Dabov_Foi_Katkovnik_2007}
K.~Dabov, A.~Foi, V.~Katkovnik, and K.~Egiazarian,
\newblock ``Image denoising by sparse {3D} transform-domain collaborative
  filtering,''
\newblock {\em IEEE Trans. Image Process.}, vol. 16, no. 8, pp. 2080--2095,
  Aug. 2007.

\bibitem{Venkatakrishnan_Bouman_Wohlberg_2013}
S.~Venkatakrishnan, C.~Bouman, and B.~Wohlberg,
\newblock ``Plug-and-play priors for model based reconstruction,''
\newblock in {\em Proc. IEEE Global Conference on Signal and Information
  Processing}, 2013, pp. 945--948.

\bibitem{Sreehari_Venkatakrishnan_Wohlberg_2015}
S.~Sreehari, S.~V. Venkatakrishnan, B.~Wohlberg, L.~F. Drummy, J.~P. Simmons,
  and C.~A. Bouman,
\newblock ``Plug-and-play priors for bright field electron tomography and
  sparse interpolation,''
\newblock Available online: http://arxiv.org/abs/1512.07331, Dec. 2015.

\bibitem{Metzler_Maleki_Baraniuk_2014}
C.~A. Metzler, A.~Maleki, and R.~G. Baraniuk,
\newblock ``From denoising to compressed sensing,''
\newblock Available online: http://arxiv.org/abs/1406.4175, Jul. 2014.

\bibitem{Ma_Zhu_Baron_2015}
Y.~Ma, J.~Zhu, and D.~Baron,
\newblock ``Approximate message passing with universal denoising,''
\newblock Available online: http://arxiv.org/abs/1506.02693, Jun. 2015.

\bibitem{Dar_Bruckstein_Elad_2015}
Y.~Dar, A.~M. Bruckstein, M.~Elad, and R.~Giryes,
\newblock ``Postprocessing of compressed images via sequential denoising,''
\newblock Available online: http://arxiv.org/pdf/1510.09041v1.pdf, Oct. 2015.

\bibitem{Milanfar_2013b}
P.~Milanfar,
\newblock ``Symmetrizing smoothing filters,''
\newblock {\em SIAM Journal on Imaging Sciences}, vol. 6, no. 1, pp. 263--284,
  2013.

\bibitem{Chan_Zickler_Lu_2015}
S.~H. Chan, T.~Zickler, and Y.~M. Lu,
\newblock ``Understanding symmetric smoothing filters via {Gaussian}
  mixtures,''
\newblock in {\em Proc. IEEE Intl. Conf. Image Process.}, Sep 2015.

\bibitem{Chan_Zickler_Lu_2016}
S.~H. Chan, T.~Zickler, and Y.~M. Lu,
\newblock ``Demystifying symmetric smoothing filters,''
\newblock Available online: http://arxiv.org/abs/1601.00088, Jan. 2016.

\bibitem{Milanfar_2013a}
P.~Milanfar,
\newblock ``A tour of modern image filtering,''
\newblock {\em {IEEE} Signal Processing Magazine}, vol. 30, pp. 106--128, Jan.
  2013.

\bibitem{Meyer_Shen_2012}
F.~Meyer and X.~Shen,
\newblock ``Perturbation of the eigenvectors of the graph {Laplacian}:
  Application to image denoising,''
\newblock {\em Applied and Computational Harmonic Analysis}, 2013,
\newblock In press. Available online at http://arxiv.org/abs/1202.6666.

\bibitem{Sinkhorn_Knopp_1967}
R.~Sinkhorn and P.~Knopp,
\newblock ``Concerning non-negative matrices and doubly-stochastic matrices,''
\newblock {\em Pacific Journal of Mathematics}, vol. 21, pp. 343 -- 348, 1967.

\bibitem{Shuman_Narang_Frossard_2013}
D.~I. Shuman, S.~K. Narang, P.~Frossard, A.~Ortega, and P.~Vandergheynst,
\newblock ``The emerging field of signal processing on graphs: Extending
  high-dimensional data analysis to networks and other irregular domains,''
\newblock {\em IEEE Signal Process. Mag.}, vol. 30, no. 3, pp. 1053--5888, Apr.
  2013.

\bibitem{Talebi_Milanfar_2014}
H.~Talebi and P.~Milanfar,
\newblock ``Global image denoising,''
\newblock {\em IEEE Trans. Image Process.}, vol. 23, no. 2, pp. 755--768, Feb.
  2014.

\bibitem{Talebi_Zhu_Milanfar_2013}
H.~Talebi, X.~Zhu, and P.~Milanfar,
\newblock ``How to {SAIF-ly} boost denoising performance,''
\newblock {\em IEEE Trans. Image Process.}, vol. 22, no. 4, pp. 1470--1485,
  Apr. 2013.

\bibitem{Shepard_1968}
D.~Shepard,
\newblock ``A two-dimensional interpolation function for irregularly-spaced
  data,''
\newblock in {\em Proc. ACM National Conference}, 1968, pp. 517--524.

\end{thebibliography}

\end{document}